\documentclass[11pt,a4paper]{article}
\pdfoutput=1
\usepackage[hyperref]{acl2020}
\usepackage{times}
\usepackage{latexsym}

\usepackage{amsmath}
\usepackage{xspace}
\usepackage{amsfonts}
\usepackage{amssymb}
\usepackage{amsthm}
\usepackage{csquotes}
\usepackage{booktabs}
\usepackage{tipa}
\usepackage[T5,T1]{fontenc}
\usepackage{combelow}

\DeclareTextSymbolDefault{\ohorn}{T5}
\DeclareTextSymbolDefault{\uhorn}{T5}

\usepackage{xfrac}

\newcommand{\bert}{\textsc{bert}\xspace}
\newcommand{\invbert}{\textsc{bert}^{-1}\xspace}
\newcommand{\id}{\texttt{id}\xspace}
\newcommand{\R}{\mathbb{R}}
\newcommand{\MI}{\mathrm{I}}
\newcommand{\ent}{\mathrm{H}}
\newcommand{\KL}{\mathrm{KL}}

\newcommand{\relu}{\mathrm{ReLU}}

\newcommand{\vs}{\mathbf{s}}

\newcommand{\vr}{\mathbf{r}}
\newcommand{\vt}{\mathbf{t}}

\newcommand{\vc}{\mathbf{c}}
\newcommand{\ve}{\mathbf{e}}
\newcommand{\defn}[1]{\textbf{#1}}

\newcommand{\calT}{\mathcal{T}}

\newcommand{\calG}{\mathcal{G}}

\newcommand{\vocab}{\mathcal{V}}
\renewcommand{\th}{\text{th}}
\newcommand{\word}[1]{\textit{#1}}
\newcommand{\pos}[1]{\textsc{#1}}

\newcommand{\vtheta}{{\boldsymbol \theta}}
\newcommand{\qtheta}{q_{\vtheta}}
\newcommand{\qthetaone}{q_{\vtheta1}}
\newcommand{\qthetatwo}{q_{\vtheta2}}
\newcommand{\entqtheta}{\ent_{\qtheta}}
\newcommand{\entqthetaone}{\ent_{\qthetaone}}
\newcommand{\entqthetatwo}{\ent_{\qthetatwo}}
\newcommand{\KLqthetaone}{\KL_{\qthetaone}}
\newcommand{\KLqthetatwo}{\KL_{\qthetatwo}}
\newtheorem{prop}{Proposition}

\newtheorem{assumption}{Assumption}
\newtheorem{cor}{Corollary}
\newtheorem{thm}{Theorem}

\usepackage{cleveref}
\crefname{section}{\S}{\S\S}
\Crefname{section}{\S}{\S\S}
\crefname{table}{Table}{}
\crefname{figure}{Figure}{}
\crefname{algorithm}{Algorithm}{}
\crefname{equation}{eq.}{}
\crefname{appendix}{App.}{}
\crefname{thm}{Theorem}{}
\crefname{prop}{Proposition}{}
\crefname{cor}{Corollary}{}
\crefname{observation}{Observation}{}
\crefname{assumption}{Assumption}{}
\crefformat{section}{\S#2#1#3}

\usepackage{bbm}
\usepackage{microtype}

\aclfinalcopy %
\def\aclpaperid{1077} %

\everypar{\looseness=-1}

\title{Information-Theoretic Probing for Linguistic Structure}

\newcommand{\ucambridge}{\normalfont \text{\textipa{D}}}
\newcommand{\ethz}{\text{\normalfont \textipa{Q}}}
\newcommand{\fairesearch}{\normalfont \text{\textipa{@}}}

\author{Tiago Pimentel$^{\ucambridge}$ Josef Valvoda$^{\ucambridge}$ Rowan Hall Maudslay$^{\ucambridge}$ Ran Zmigrod$^{\ucambridge}$ \\
\textbf{Adina Williams$^{\fairesearch}$ Ryan Cotterell$^{\ucambridge,\ethz}$} \\
  $^{\ucambridge}$University of Cambridge~\;~$^{\fairesearch}$Facebook AI Research~\;~%
  $^{\ethz}$ETH Z\"{u}rich \\
  \texttt{tp472@cam.ac.uk},~\;~ \texttt{jv406@cam.ac.uk},~\;~ \texttt{rh635@cam.ac.uk}, \\
  \texttt{rz279@cam.ac.uk},~\;~ \texttt{adinawilliams@fb.com},~\;~ \texttt{rdc42@cl.cam.ac.uk}
}

\date{}

\begin{document}
\maketitle
\begin{abstract}
The success of neural networks on a diverse set of NLP tasks has led researchers to question how much these networks actually ``know'' about natural language.
Probes are a natural way of assessing this.
When probing, a researcher chooses a linguistic task and trains a supervised model to predict annotations in that linguistic task from the network's learned representations.
If the probe does well, the researcher may conclude that the representations encode knowledge related to the task. 
A commonly held belief is that using simpler models as probes is better; the logic is that simpler models will \emph{identify linguistic structure}, but not \emph{learn the task itself}.
We propose an information-theoretic operationalization of probing as estimating mutual information that contradicts this received wisdom: one should always select the highest performing probe one can, even if it is more complex, since it will result in a tighter estimate, and thus reveal more of the linguistic information inherent in the representation.
The experimental portion of our paper focuses on empirically estimating the mutual information
between a linguistic property and BERT, comparing these estimates
to several baselines. We evaluate on a set of ten typologically diverse languages often underrepresented in NLP research---plus English---totalling eleven languages. 
Our implementation is available in \url{https://github.com/rycolab/info-theoretic-probing}.
\end{abstract}

\section{Introduction}
Neural networks are the backbone of modern state-of-the-art natural language processing (NLP) systems.
One inherent by-product of training a neural network is the production of real-valued representations.
Many speculate that these representations encode a continuous analogue of discrete linguistic properties, e.g.,\ part-of-speech tags, due to the networks' impressive performance on many NLP tasks~\citep{belinkov-etal-2017-evaluating}. 
As a result of this speculation, one common thread of research focuses on the construction of \textbf{probes}, i.e., supervised models that are trained to extract the linguistic properties directly~\citep{belinkov-etal-2017-evaluating,conneau-etal-2018-cram, peters-etal-2018-dissecting,zhang-bowman-2018-language, naik-etal-2018-stress,tenney2019}. A syntactic probe, then, is a model for extracting syntactic properties, such as part of speech, from the representations \cite{hewitt-liang-2019-designing}.

In this work, we question what the goal of probing for linguistic properties ought to be.
Informally, probing is often described as an attempt to discern how much information representations encode about a specific linguistic property.
We make this statement more formal: We assert that the natural operationalization of probing is estimating the mutual information \cite{cover-thomas} between a representation-valued random variable and a linguistic property--valued random variable.
This operationalization gives probing a clean, information-theoretic foundation, and allows us to consider what ``probing'' actually means. 

Our analysis also provides insight into how to choose a probe family: We show that choosing the highest-performing probe, independent of its complexity, is optimal for achieving the best estimate of mutual information (MI).
This contradicts the received wisdom that one should always select simple probes over more complex ones \cite{alain2016understanding,liu-etal-2019-linguistic,hewitt-manning-2019-structural}.
In this context, we also discuss the recent work of \newcite{hewitt-liang-2019-designing} who proposes \defn{selectivity} as a criterion for choosing families of probes.
\newcite{hewitt-liang-2019-designing} defines selectivity as the performance difference between a probe on the target task and a control task, writing ``[t]he selectivity of a probe puts linguistic task accuracy in context with the probe's capacity to memorize from word types.''
They further ponder: ``when a probe achieves high accuracy on a linguistic task using a representation, can we conclude that the representation encodes linguistic structure, or has the probe just learned the task?''
Information-theoretically, there is \emph{no difference} between learning the task and probing for linguistic structure, as we will show; thus, it follows that one should always employ the best possible probe for the task without resorting to artificial constraints.  

In the experimental portion of the paper, we empirically analyze word-level part-of-speech labeling, a common syntactic probing task \cite{hewitt-liang-2019-designing, DBLP:journals/corr/abs-1903-09442}, within our MI operationalization.
Working on a typologically diverse set of languages (Basque, Czech, English, Finnish, Indonesian, Korean, Marathi, Tamil, Telugu, Turkish and Urdu), we
show that only in five of these eleven languages do we recover higher estimates of mutual information between part-of-speech tags
and BERT \citep{devlin-etal-2019-bert}, a common contextualized embedder, than from a control. 
These modest improvements suggest that most of the information needed to tag part-of-speech well is encoded at the lexical level, and does not require sentential context.
Put more simply, words are not very ambiguous with respect to part of speech, a result known to practitioners of NLP \cite{garrette-etal-2013-real}.
We interpret this to mean that part-of-speech labeling is not a very informative probing task. 
We further investigate how BERT fares in dependency labeling, as analysed by \citet{tenney2019}. 
In this task, estimates based on BERT return more information than a type-level embedding in all analysed languages. 
However, our MI estimates still only show that BERT contains at most $12\%$ more information than the control.

We also remark that operationalizing probing information-theoretically gives us a simple, but stunning result: contextual word embeddings, e.g., BERT \cite{devlin-etal-2019-bert} and ELMo \cite{peters-etal-2018-deep}, contain \emph{the same} amount of information about the linguistic property of interest as the original sentence.
This follows from the data-processing inequality under a very mild assumption.
What this suggests is that, in a certain sense, probing for linguistic properties in representations may not be a well grounded enterprise at all. It also highlights the need to more formally define \emph{ease of extraction}.

\section{Word-Level Syntactic Probes for Contextual Embeddings} \label{sec:probescontextualembeddings}

Following \newcite{hewitt-liang-2019-designing}, we consider probes that examine syntactic knowledge in contextualized embeddings. 
These probes only consider a token's embedding in isolation, and try to perform the task using only that information.
Specifically, in this work, we consider part-of-speech (POS) and dependency labeling: determining a word's part of speech in a given sentence and the dependency relation for a pair of tokens joined by a dependency arc.
Say we wish to determine whether the word \word{love} is a \pos{noun} or a \pos{verb}.
This task requires the sentential context for success.
As an example, consider the utterance ``love is blind'' where, only with the context, is it clear that \word{love} is a \pos{noun}. 
Thus, to do well on this task, the contextualized embeddings need to encode enough about the surrounding context to correctly guess the POS. 
Analogously, we need the whole sentence to know that \word{love} is the \pos{nominal subject}. Whereas in the sentence ``greed can blind love'', \word{love} is the \pos{direct object}. %

\subsection{Notation}
Let $S$ be a random variable ranging over all possible sequences of words. For the sake of this paper, we  assume the vocabulary $\vocab$ is finite and, thus, the values $S$ can take are in $\vocab^*$. 
We write $\vs \in S$ as $\vs = s_1 \cdots s_{|\vs|}$ for a specific sentence, where each $s_i \in \vocab$ is a specific token in the sentence at the position $i \in \mathbb{Z}_{+}$.
We also define the random variable $W$ that ranges over the vocabulary $\vocab$.
We define both a sentence-level random variable $S$ and a word type-level random variable $W$ since each will be useful in different contexts during our exposition.

Next, let $T$ be a random variable whose possible values are the analyses $t$ that we want to consider for token $s_i$ in its sentential context, $\vs = s_1 \cdots s_i \cdots s_{|\vs|}$. 
In the discussion, we focus on predicting the part-of-speech tag of the $i^\text{th}$ word $s_i$, but the same results apply to the dependency label of an edge between two words.
We denote the set of values $T$ can take as the set $\calT$.
Finally, let $R$ be a representation-valued random variable for a token $s_i$ derived from the entire sentence $\vs$.
We write $\vr \in \R^d$ for a value of $R$. 
While any given value $\vr$ is a continuous vector, there are only a countable number of values $R$ can take.\footnote{In this work, we ignore the fact that the floating points have precision constraints in practice.}
To see this, note there are only a countable number of sentences in $\vocab^*$.\looseness=-1

Next, we assume there exists a true distribution $p(t, \vs, i)$ over analyses $t$ (elements of $\calT$),  sentences $\vs$ (elements of $\vocab^*$), and positions $i$ (elements of $\mathbb{Z}_{+}$).
Note that the conditional distribution $p(t \mid \vs, i)$ gives us the true distribution over analyses $t$ for the $i^{\th}$ word token in the sentence $\vs$.
We will augment this distribution such that $p$ is additionally a distribution over $\vr$, i.e.,
\begin{align}\label{eq:true}
    p(\vr, t, \vs, i) = \delta(\vr \mid \vs, i)\,p(t, \vs, i)  
\end{align}
where we define the augmentation as:
\begin{align}
    \delta(\vr \mid \vs, i) &= \mathbbm{1}\{\vr=\bert(\vs)_i\}
\end{align}
Since contextual embeddings are a deterministic function of a sentence $\vs$, the augmented distribution in \cref{eq:true} has no more randomness than the original---its entropy is the same.
We assume the values of the random variables defined above are distributed according to this (unknown) $p$.
While we do not have access to $p$, we assume the data in our corpus were drawn according to it. 
Note that $W$---the random variable over possible word types---is distributed according to
\begin{equation}
    p(w) = \sum_{\vs \in \vocab^*}\sum_{i=1}^{|\vs|} \delta(w \mid \vs, i)\,p(\vs, i)
\end{equation}
where we define the deterministic distribution
\begin{equation}
    \delta(w \mid \vs, i) = \mathbbm{1}\{\vs_i = w\}
\end{equation}

\subsection{Probing as Mutual Information}\label{subsec:probeasMI}
The task of supervised probing is an attempt to ascertain how much information a specific representation $\vr$ tells us about the value of $t$.
This is naturally operationalized as the mutual information, a quantity from information theory:
\begin{align}
   \MI(T; R) = \ent(T) - \ent(T \mid R)
\end{align}
where we define the entropy, which is constant with respect to the representations, as
\begin{align}
    \ent(T) = &-  \sum_{t \in \calT} p(t) \log p(t)
\end{align}
and we define the conditional entropy as
\begin{align}
    &\ent(T \mid R) = \int p(\vr) \,\ent\left(T \mid R = \vr \right)\mathrm{d}\vr \\
     &=\sum_{\vs \in \vocab^*} \sum_{i = 1}^{|\vs|} p(\vs, i) \,\ent\left(T \mid R = \bert(\vs)_i \right) \nonumber
\end{align}
where the point-wise conditional entropy inside the sum is defined as 
\begin{equation}
\ent(T \mid R = \vr) = -\sum_{t \in \calT} p(t \mid \vr) \log p(t \mid \vr)
\end{equation}
Again, we will not know any of the distributions required to compute these quantities; the distributions in the formulae are marginals and conditionals of the true distribution discussed in \cref{eq:true}.

\subsection{Bounding Mutual Information}
The desired conditional entropy, $\ent(T \mid R)$ is not readily available, but with a model $\qtheta (\vt \mid \vr)$ in hand, we can upper-bound it by measuring their empirical cross entropy:
\begin{align}\label{eq:ent_estimate}
    \ent&(T\mid R) := -\underset{{(t, \vr) \sim p(\cdot, \cdot)}}{\mathbb{E}} \left[ \log p(t \mid \vr) \right] \\
    &= - \underset{{(t, \vr) \sim p(\cdot, \cdot)}}{\mathbb{E}} \left[ \log \frac{p(t \mid \vr) \qtheta(t \mid \vr)}{\qtheta(t \mid \vr)} \right] \nonumber \\
    &= - \underset{{(t, \vr) \sim p(\cdot, \cdot)}}{\mathbb{E}} \left[ \log \qtheta(t \mid \vr) + \log \frac{p(t \mid \vr)}{\qtheta(t \mid \vr)} \right] \nonumber \\
    &= \underbrace{\ent_{\qtheta}(T \mid R)}_{\textit{estimate}} - \underbrace{\underset{\vr \sim p(\cdot)}{\mathbb{E}} \KL(p(\cdot \mid \vr) \mid \mid \qtheta(\cdot \mid \vr))}_{\textit{expected estimation error}} \nonumber
\end{align}
where $\ent_{\qtheta}(T \mid R)$ is the cross-entropy we obtain by using $\qtheta$ to get this estimate.
Since the KL divergence is always positive, we may lower-bound the desired mutual information 
\begin{align}
    \MI(T; R) &:= \ent(T) - \ent(T \mid R) \nonumber \\
              &\geq \ent(T) - \ent_{\qtheta}(T \mid R)
\end{align}
This bound gets tighter, the more similar---in the sense of the KL divergence---$\qtheta(\cdot \mid \vr)$ is to the true distribution $p(\cdot \mid \vr)$.

\paragraph{Bigger Probes are Better.}
If we accept mutual information as a natural operationalization for how much representations encode a target linguistic task (\S \ref{subsec:probeasMI}), the best estimate of that mutual information is the one where the probe $\qtheta(t \mid \vr)$ is best at the target task.
In other words, we want the best probe $\qtheta(t \mid \vr)$ such that we get the tightest bound to the actual distribution $p(t\mid \vr)$.
This paints the question posed in \newcite{hewitt-liang-2019-designing}, who write
\begin{displayquote}
``when a probe achieves high accuracy on a linguistic task using a representation, can we conclude that the representation encodes linguistic structure, or has the probe just learned the task?''
\end{displayquote}
as a false dichotomy.\footnote{Assuming that the authors intended `or' here as strictly non-inclusive. See \citet[91]{levinson2000} and \citet[1743]{chevallier2008} on conversational implicatures from `or'.}
From an information-theoretic view, we will always prefer the probe that does better at the target task, since there is \emph{no difference} between learning a task and the representations encoding the linguistic structure.

\section{Control Functions}\label{sec:control-functions}
To place the performance of a probe in perspective, \newcite{hewitt-liang-2019-designing} develops the notion of a control task.
Inspired by this, we develop an analogue we term \defn{control functions}, which are functions of the representation-valued random variable $R$.
Similar to \newcite{hewitt-liang-2019-designing}'s control tasks, the goal of a control function $\vc(\cdot)$ is to place the mutual information $\MI(T; R)$ in the context of a baseline that the control function encodes.
Control functions have their root in the data-processing inequality \cite{cover-thomas}, which states that, for any function $\vc(\cdot)$, we have
\begin{equation}\label{eq:data-processing}
    \MI(T; R) \geq \MI(T; \vc(R))
\end{equation}
In other words, information can only be lost by processing data.
A common adage associated with this inequality is ``garbage in, garbage out.'' 

\subsection{Type-Level Control Functions}
We focus on type-level control functions in this paper. These functions have the effect of decontextualizing the embeddings, being related to the common trend of analyzing probe results in comparison to input layer embeddings \citep{belinkov2017analyzing,liu-etal-2019-linguistic,hewitt-manning-2019-structural,tenney2019}.
Such functions allow us to inquire how much the contextual aspect of the contextual embeddings help the probe perform the target task.
To show that we may map from contextual embeddings to the identity of the word type, we need the following assumption.
\begin{assumption}\label{ass:one}
Every contextualized embedding is unique, i.e., for any pair of sentences $\vs, \vs' \in \vocab^*$, we have $(\vs \neq \vs') \mid\mid (i \neq j) \Rightarrow \bert(\vs)_i \neq \bert(\vs')_j$ for all $i \in \{1, \ldots |\vs|\}$ and $j \in \{1, \ldots, |\vs'|\}$.
\end{assumption}
We note that \cref{ass:one} is mild.
Contextualized word embeddings map words (in their context) to $\R^d$, which is an uncountably infinite space.
However, there are only a countable number of sentences, which implies only a countable number of sequences of real vectors in $\R^d$ that a contextualized embedder may produce.
The event that any two embeddings would be the same across two distinct sentences is infinitesimally small.\footnote{Indeed, even if we sampled every embedding randomly from a $d$-dimensional Gaussian, the probability that we would ever sample the same real vector is zero.}
\cref{ass:one} yields the following corollary.
\begin{cor}\label{cor:one}
There exists a function $\emph{\id} : \R^d \rightarrow V$ that maps a contextualized embedding to its word type. The function $\emph{\id}$ is not a bijection since multiple embeddings will map to the same type.
\end{cor}
Using \cref{cor:one}, we can show that any \emph{non-contextualized} word embedding will contain \emph{no more} information than a contextualized word embedding.
More formally, we do this by constructing a look-up function $\ve : V \rightarrow \R^d$ that maps a word to a word embedding.
This embedding may be one-hot, randomly generated ahead of time, or the output of a data-driven embedding method, e.g. fastText \cite{fasttext}. 
We can then construct a control function as the composition of the look-up function $\ve$ and the id function $\id$.
Using the data-processing inequality, we can prove that in a word-level prediction task, any non-contextual (type level) word-embedding will contain no more information than a contextualized (token level) one, such as BERT and ELMo.
Specifically, we have
\begin{align}
    \MI(T ; R) &\geq \\ 
    &\MI(T ; \id(R)) = \MI(T; W) \geq \MI(T; \ve(W)) \nonumber
\end{align}
This result\footnote{Note that although this result holds in theory, in practice the functions $\id$ and $\ve(\cdot)$ might be arbitrarily hard to estimate. This is discussed in length in \cref{sec:ease-extract}.} is intuitive and, perhaps, trivial---context matters information-theoretically.
However, it gives us a principled foundation by which to measure the effectiveness of probes as we will show in \cref{sec:gain}. 

\subsection{How Much Information Did We Gain?}\label{sec:gain}
We will now quantify how much a contextualized word embedding knows about a task with respect to a specific control function $\vc(\cdot)$.
We term how much more information the contextualized embeddings have about a task than a control variable the \defn{gain}, $\calG$, which we define as
\begin{align}\label{eq:gain}
    \calG(T, &\,R, \vc) = \MI(T; R) - \MI(T; \vc(R)) \\ 
    &= \ent(T \mid \vc(R)) - \ent(T \mid R) \geq 0 \nonumber
\end{align}
The gain function will be our method for measuring how much more information contextualized representations have over a controlled baseline, encoded as the function $\vc$.
We will empirically estimate this value in \cref{sec:experiments}.
Interestingly enough, the gain has a straightforward interpretation.
\begin{prop}\label{prop:interpretation}
The gain function is equal to the following conditional mutual information
\begin{equation}
\MI(T; R \mid \vc(R)) = \calG(T, R, \vc)
\end{equation}
\end{prop}
\begin{proof}
\begin{align*}%
    \!\MI(T; R \mid \vc(R)) &:= \MI(T; R) - \MI(T; R; \vc(R)) \\
    &= \MI(T; R) - \MI(T; \vc(R))  \\
    &= \calG(T, R, \vc)
\end{align*}
The jump from the first to the second equality follows since $R$ encodes, by construction, all the information about $T$ provided by $\vc(R)$.
\end{proof}
\cref{prop:interpretation} gives us a clear understanding of the quantity we wish to estimate: It is how much information about a task is encoded in the representations, given some control knowledge.
If properly designed, this control transformation will remove information from the probed representations.

\subsection{Approximating the Gain}
The gain, as defined in \cref{eq:gain}, is intractable to compute.
In this section we derive a pair of variational bounds on $\calG(T, R, \ve)$---one upper and one lower.
To approximate the gain, we will simultaneously minimize an upper and maximize a lower-bound on \cref{eq:gain}.
We begin by approximating the gain in the following manner
\begin{align}\label{eq:approx}
     \calG(T, R, \ve) &\approx \\
     &\underbrace{\entqthetatwo(T \mid \vc(R)) - \entqthetaone(T \mid R)}_{\textit{estimated }\calG_{\qtheta}(T, R, \ve)}  \nonumber
\end{align}
these cross-entropies can be empirically estimated.
We will assume access to a corpus $\{(t_i, \vr_i)\}_{i=1}^N$ that is human-annotated for the target linguistic property; we further assume that these are samples $(t_i, \vr_i) \sim p(\cdot, \cdot)$ from the true distribution. 
This yields a second approximation that is tractable:
\begin{equation}
     \entqtheta(T ; R) \approx -\frac{1}{N} \sum\limits_{i=1}^{N} \log \qtheta(t_i \mid \vr_i)
\end{equation}
This approximation is exact in the limit $N \rightarrow \infty$ by the law of large numbers. 

We note the approximation given in \cref{eq:approx} may be either positive or negative and its estimation error follows from \cref{eq:ent_estimate}:
\begin{align}\label{eq:approx_error}
    \Delta &\,= \underset{\vr \sim p(\cdot)}{\mathbb{E}} \KL(p(\cdot \mid \vr) \mid \mid \qthetaone(\cdot \mid \vr)) \\
        &- \underset{\vr \sim p(\cdot)}{\mathbb{E}} \KL(p(\cdot \mid \vc(\vr)) \mid \mid \qthetatwo(\cdot \mid \vc(\vr))) \nonumber \\
        &\,= \KLqthetaone(T, R) - \KLqthetatwo(T, \vc(R)) \nonumber
\end{align}
where we abuse the KL notation to simplify the equation.
This is an undesired behavior since we know the gain itself is non-negative by the data-processing inequality, but we have yet to devise a remedy.\looseness=-1

We justify the approximation in \cref{eq:approx} with a pair of variational bounds.
The following two corollaries are a result of \cref{thm:variational_bounds} in \cref{appendix:a}.
\begin{cor}\label{cor:upper}
We have the following upper-bound on the gain
\begin{align}\label{eq:upper}
    \calG(T, R,\ve) & \\
    &\leq \calG_{\qtheta}(T, R,\ve)\!+\! \KLqthetaone(T, R) \nonumber
\end{align}
\end{cor}

\begin{cor}\label{cor:lower}
We have the following lower-bound on the gain
\begin{align}\label{eq:lower}
    \calG(T, R,\ve) & \\
    &\geq \calG_{\qtheta}(T, R,\ve) -  \KLqthetatwo(T, \vc(R)) \nonumber
\end{align}
\end{cor}
\noindent The conjunction of \cref{cor:upper} and \cref{cor:lower} suggest a simple procedure for finding a good approximation: We choose $\qthetaone(\cdot \mid r)$ and $\qthetatwo(\cdot \mid r)$ so as to \emph{minimize} \cref{eq:upper} and \emph{maximize} \cref{eq:lower}, respectively.
These distributions contain no overlapping parameters, by construction, so these two optimization routines may be performed independently.
We will optimize both with a gradient-based procedure, discussed in \cref{sec:experiments}.

\section{Understanding Probing Information-Theoretically}\label{sec:we-hate-probing}
In \cref{sec:control-functions}, we developed an information-theoretic framework for thinking about probing contextual word embeddings for linguistic structure.
However, we now cast doubt on whether probing makes sense as a scientific endeavour.
We prove in \cref{sec:context} that contextualized word embeddings, by construction, contain no more information about a word-level syntactic task than the original sentence itself.
Nevertheless, we do find a meaningful scientific interpretation of control functions.
We expound upon this in \cref{sec:control-functions-meaning}, arguing that control functions are useful, not for understanding representations, but rather for understanding the influence of sentential context on word-level syntactic tasks, e.g., labeling words with their part of speech.

\subsection{You Know Nothing, BERT}\label{sec:context}
To start, we note the following corollary
\begin{cor}
It directly follows from \cref{ass:one} that $\bert$ is a bijection between sentences $\vs$ and sequences of embeddings $\langle \vr_1, \ldots, \vr_{|\vs|} \rangle$. 
As $\bert$ is a bijection, it has an inverse, which we will denote as $\invbert$.
\end{cor}
\begin{thm}\label{thm:bert}
$\bert(S)$ cannot provide more information about $T$ than the sentence $S$ itself.
\end{thm}
\begin{proof}
\begin{align}
    \MI(T ; S) &\geq \MI(T; \bert(S)) \\
               &\geq \MI(T; \invbert(\bert(S))) \nonumber \\
               &= \MI(T ; S) \nonumber 
\end{align}
This implies $\MI(T ; S) = \MI(T; \bert(S))$.\footnote{Actually, \citeauthor{hewitt-liang-2019-designing} likely had an intuition about this in mind when they wrote ``[a] sufficiently expressive probe with enough training data could learn \emph{any} task on top of it'' \citep{hewitt-liang-2019-designing}.} 
This is not a BERT-specific result---it rests on the fact that the data-processing inequality is tight for bijections. 
\end{proof} 
While \cref{thm:bert} is a straightforward application of the data-processing inequality, it has deeper ramifications for probing.
It means that if we search for syntax in the contextualized word embeddings of a sentence, we should not expect to find any more syntax than is present in the original sentence.
In a sense, \cref{thm:bert} is a cynical statement: under our operationalization, the endeavour of finding syntax in contextualized embeddings sentences is nonsensical.
This is because, under \cref{ass:one}, we know the answer \textit{a priori}---the contextualized word embeddings of a sentence contain exactly the same amount of information about syntax as does the sentence itself.

\subsection{What Do Control Functions Mean?}\label{sec:control-functions-meaning}
Information-theoretically, the interpretation of control functions is also interesting.
As previously noted, our interpretation of control functions in this work does not provide information about the representations themselves.
Indeed, the same reasoning used in \cref{cor:one} can be used to devise a function $\id_s(\vr)$ which maps a contextual representation of a token back to its sentence.
For a type-level control function $\vc$, by the data-processing inequality, we have that $\MI(T; W) \geq \MI(T; \vc(R))$.
Consequently, we can get an upper-bound on how much information we can get out of a decontextualized representation.
If we assume we have perfect probes, then we get that the true gain function is $\MI(T; S) - \MI(T; W) = \MI(T; S \mid W)$.
This quantity is interpreted as the amount of knowledge we gain about the word-level task $T$ by knowing $S$ (i.e., the sentence) in addition to  $W$ (i.e., the word type).
Therefore, a perfect probe provides insights about language and not about the actual representations.

\subsection{Discussion: Ease of Extraction} \label{sec:ease-extract}
We do acknowledge another interpretation of the work of \newcite{hewitt-liang-2019-designing} \textit{inter alia}; BERT makes the syntactic information present in an ordered sequence of words more easily extractable.
However, ease of extraction is not a trivial notion to operationalize, and indeed, we know of no attempt to do so;\footnote{\newcite{Xu2020A} is a possible exception.} it is certainly more complex to determine than the number of layers in a multi-layer perceptron (MLP).
Indeed, a MLP with a single hidden layer can represent any function over the unit cube, with the caveat that we may need a very large number of hidden units \cite{cybenko1989approximation}.

Although for perfect probes the above results should hold, in practice $\id(\cdot)$ and $\vc(\cdot)$ may be hard to approximate.
Furthermore, if these functions were to be learned, they might require an unreasonably large dataset.
Learning a random embedding control function, for example, would require a dataset containing all words in the vocabulary $V$---in an open vocabulary setting an infinite dataset would be required!
``Better'' representations should make their respective probes easily learnable---and consequently their encoded information is more accessible \citep{voita2020information}.

We suggest that future work on probing should focus on operationalizing ease of extraction more rigorously---even though we do not attempt this ourselves.
As previously argued by \newcite[\S 5]{saphra-lopez-2019-understanding}, the advantage of simple probes is that they may reveal something about the \emph{structure} of the encoded information---i.e., is it structured in such a way that it can be easily taken advantage of by downstream consumers of the contextualized embeddings?
Many researchers who are interested in less complex probes have, either implicitly or explicitly, had this in mind.

\section{A Critique of Control Tasks}
We agree with \newcite{hewitt-liang-2019-designing}---and with both \citet{zhang-bowman-2018-language} and \citet{tenney2019}---that we should have controlled baselines when probing for linguistic properties. However, we disagree with parts of their methodology for constructing control tasks. We present these disagreements here. 

\subsection{Structure and Randomness}
\newcite{hewitt-liang-2019-designing} introduces control tasks to evaluate the effectiveness of probes.
We draw inspiration from this technique as evidenced by our introduction of control functions.
However, we take issue with the suggestion that controls should have \defn{structure} and \defn{randomness}, to use the terminology from \newcite{hewitt-liang-2019-designing}.
They define structure as ``the output for a word token is a deterministic function of the word type.''
This means that they are stripping the language of ambiguity with respect to the target task.
In the case of part-of-speech labeling, \word{love} would either be a \pos{noun} or a \pos{verb} in a control task, never both: this is a problem. The second feature of control tasks is randomness, i.e., ``the output for each word type is sampled independently at random.'' %
In conjunction, structure and randomness may yield a relatively trivial task that does not look like natural language.

What is more, there is a closed-form solution for an optimal, retrieval-based ``probe'' that has zero learned  parameters: %
If a word type appears in the training set, return the label with which it was annotated there, otherwise return the most frequently occurring label across all words in the training set.
This probe will achieve an accuracy that is 1 minus the out-of-vocabulary rate (the number of tokens in the test set that correspond to novel types divided by the number of tokens) times the percentage of tags in the test set that do not correspond to the most frequent tag (the error rate of the guess-the-most-frequent-tag classifier).
In short, the best model for a control task is a pure memorizer that guesses the most frequent tag for out-of-vocabulary words.

\subsection{What's Wrong with Memorization?}
\newcite{hewitt-liang-2019-designing} proposes that probes should be optimized to maximize accuracy \emph{and} selectivity.
Recall selectivity is given by the distance between the accuracy on the original task and the accuracy on the control task using the same architecture.
Given their characterization of control tasks, maximising selectivity leads to a selection of a model that is bad at memorization.
But why should we punish memorization?
Much of linguistic competence is about generalization, however memorization also plays a key role \citep{fodor1974, nooteboom2002storage, fromkin2018}, with word learning \citep{carey1978} being an obvious example.
Indeed, maximizing selectivity as a criterion for creating probes seems to artificially disfavor this property.

\begin{table*}
    \centering
    \resizebox{\textwidth}{!}{%
    \begin{tabular}{l l l c c c c c c c c}
    \toprule
 & \multicolumn{2}{c}{\textbf{\# Tokens}} & & & \multicolumn{1}{c}{\textbf{\textsc{bert}}} & \multicolumn{2}{c}{\textbf{fastText}}  & \multicolumn{2}{c}{\textbf{one-hot}} \\ \cmidrule(r){2-3} \cmidrule(r){6-6} \cmidrule(r){7-8} \cmidrule(r){9-10}
Language & Train & Test & \# POS & $\ent(T)$ & $\ent(T \mid R)$  & $\ent(T \mid \vc(R))$ & $\calG(T, R, \vc)$ & $\ent(T \mid \vc(R))$ & $\calG(T, R, \vc)$ \\
    \midrule
Basque & $\phantom{0}\phantom{0}72{,}869$ & $\phantom{0}24{,}335$ & 15 & 3.17 & 0.36 & 0.29 & -0.06 (-2.0\%) & 0.80 & 0.44 (14.0\%) \\
Czech & $1{,}173{,}281$ & $173{,}906$ & 16 & 3.33 & 0.10 & 0.11 & \phantom{-}0.02 (\phantom{-}0.5\%) & 0.35 & 0.25 (\phantom{0}7.6\%) \\
English & $\phantom{0}203{,}762$ & $\phantom{0}24{,}958$ & 16 & 3.61 & 0.23 & 0.39 & \phantom{-}0.16 (\phantom{-}4.4\%) & 0.64 & 0.41 (11.4\%) \\
Finnish & $\phantom{0}162{,}584$ & $\phantom{0}21{,}078$ & 14 & 3.17 & 0.25 & 0.19 & -0.06 (-2.0\%) & 0.80 & 0.54 (17.1\%) \\
Indonesian & $\phantom{0}\phantom{0}97{,}495$ & $\phantom{0}11{,}779$ & 15 & 3.24 & 0.38 & 0.35 & -0.03 (-0.8\%) & 0.64 & 0.26 (\phantom{0}8.0\%) \\
Korean & $\phantom{0}295{,}899$ & $\phantom{0}28{,}234$ & 16 & 3.04 & 0.33 & 0.60 & \phantom{-}0.27 (\phantom{-}8.8\%) & 1.15 & 0.82 (27.0\%) \\
Marathi & $\phantom{0}\phantom{0}\phantom{0}2{,}997$ & $\phantom{0}\phantom{0}\phantom{0}412$ & 15 & 3.17 & 0.76 & 0.90 & \phantom{-}0.14 (\phantom{-}4.4\%) & 1.49 & 0.74 (23.2\%) \\
Tamil & $\phantom{0}\phantom{0}\phantom{0}6{,}329$ & $\phantom{0}\phantom{0}1{,}988$ & 13 & 3.15 & 0.58 & 0.47 & -0.11 (-3.5\%) & 1.57 & 0.99 (31.4\%) \\
Telugu & $\phantom{0}\phantom{0}\phantom{0}5{,}082$ & $\phantom{0}\phantom{0}\phantom{0}721$ & 14 & 2.73 & 0.42 & 0.42 & -0.00 (-0.1\%) & 0.93 & 0.51 (18.6\%) \\
Turkish & $\phantom{0}\phantom{0}37{,}769$ & $\phantom{0}10{,}023$ & 13 & 3.03 & 0.36 & 0.23 & -0.13 (-4.2\%) & 0.88 & 0.52 (17.1\%) \\
Urdu & $\phantom{0}108{,}674$ & $\phantom{0}14{,}806$ & 15 & 3.23 & 0.32 & 0.41 & \phantom{-}0.09 (\phantom{-}2.8\%) & 0.54 & 0.22 (\phantom{0}6.9\%) \\
    \bottomrule 
    \end{tabular}
    }
    \caption{Amount of information $\bert$, fastText or one-hot embeddings share with a POS probing task. %
    $\ent(T)$ is estimated with a plug-in estimator from same treebanks we use to train the POS labelers.} \vspace{-.5em}
    \label{tab:results-full}
\end{table*}

\subsection{What Low-Selectivity Means}
\newcite{hewitt-liang-2019-designing} acknowledges that for the more complex
task of dependency edge prediction, a MLP probe is more accurate and, therefore, preferable despite its low selectivity.
However, they offer two counter-examples where the less selective neural probe exhibits drawbacks when compared to its more selective, linear counterpart. 
We believe both examples are a symptom of using a simple probe rather than of selectivity being a useful metric for probe selection.

First, \newcite[\S 3.6]{hewitt-liang-2019-designing} point out that, in their experiments, the MLP-1 model frequently mislabels the word with suffix \textit{-s} as \pos{NNPS} on the POS labeling task.
They present this finding as a possible example of a less selective probe being less faithful in representing what linguistic information has the model learned.
Our analysis leads us to believe that, on contrary, this shows that one should be using the best possible probe to minimize the chance of misinterpreting its encoded information.
Since more complex probes achieve higher accuracy on the task, as evidence by the findings of \newcite{hewitt-liang-2019-designing}, we believe that the overall trend of misinterpretation is higher for the probes with higher selectivity.
The same applies for the second example in \citealt[\S 4.2]{hewitt-liang-2019-designing} where a less selective probe appears to be less faithful.
The paper shows that the representations on ELMo's second layer fail to outperform its word type ones (layer zero) on the POS labeling task when using the MLP-1 probe.
While the paper argues this is evidence for selectivity being a useful metric in choosing appropriate probes, we argue that this demonstrates, yet again, that one needs to use a more complex probe to minimize the chances of misinterpreting what the model has learned.
The fact that the linear probe shows a difference only demonstrates that the information is perhaps more accessible with ELMo, not that it is not present.

\section{Experiments}\label{sec:experiments}
Despite our discussion in \cref{sec:we-hate-probing}, we still wish to empirically vet our estimation technique for the gain and we use this section to highlight the need to formally define ease of extraction (as argued in \cref{sec:ease-extract}).
We consider the tasks of POS and dependency labeling, using the universal POS tag \cite{petrov-etal-2012-universal} and dependency label information from the Universal Dependencies 2.5 \cite{11234/1-3105}.
We probe the multilingual release of BERT\footnote{We used \newcite{wolf2019hugging}'s implementation.} on eleven typologically diverse languages: Basque, Czech, English, Finnish, Indonesian, Korean, Marathi, Tamil, Telugu, Turkish and Urdu; and we compute the contextual representations of each sentence by feeding it into BERT and averaging the output word piece representations for each word, as tokenized in the treebank. 

\subsection{Control Functions}
We will consider two different control functions.
Each is defined as the composition $\vc = \ve \circ \id$ with a different look-up function:
\begin{itemize}
        \item $\ve_\textit{fastText}$  returns a language specific fastText embedding \citep{fasttext};
        \item $\ve_\textit{onehot}$ returns a one-hot embedding.\footnote{We initialize random embeddings at the type level, and let them train during the model's optimization. We also experiment with fixed random embeddings---results for this control are in the Appendix.}
\end{itemize}
These functions can be considered type level, as they remove the influence of context on the word. 

\subsection{Probe Architecture}
As expounded upon above, our purpose is to achieve the best bound on mutual information we can.
To this end, we employ a deep MLP as our probe. We define the probe as 
\begin{align}
    \qtheta(t &\mid \vr) =\\
    &\text{softmax}\left( W^{(m)} \sigma\left(W^{(m-1)}\cdots \sigma(W^{(1)}\,\vr)\right)\right) \nonumber
\end{align}
an $m$-layer neural network with the non-linearity $\sigma(\cdot) = \relu(\cdot)$. 
The initial projection matrix is $W^{(1)} \in \R^{r_1 \times d}$ and the final projection matrix is $W^{(m)} \in \R^{|\calT| \times r_{m-1}}$, where $r_i=\frac{r}{2^{i-1}}$.
The remaining matrices are $W^{(i)} \in \R^{r_i \times r_{i-1}}$,
so we halve the number of hidden states in each layer.
We optimize over the hyperparameters---number of layers, hidden size, one-hot embedding size, and dropout---by using random search.
For each estimate, we train 50 models and choose the one with the best validation cross-entropy.
The cross-entropy in the test set is then used as our entropy estimate.
For dependency labeling, we follow \citet{tenney2019} and concatenate the embeddings for both a token and its head---i.e. $r=[r_i; r_{\mathrm{head}(i)}]$---as such, the initial projection matrix is actually $W^{(1)} \in \R^{r_1 \times 2 d}$.

\begin{table*}
    \centering
    \resizebox{\textwidth}{!}{%
    \begin{tabular}{l l l c c c c c c c c}
    \toprule
 & \multicolumn{2}{c}{\textbf{\# Tokens}} & & & \multicolumn{1}{c}{\textbf{\textsc{bert}}} & \multicolumn{2}{c}{\textbf{fastText}}  & \multicolumn{2}{c}{\textbf{one-hot}} \\ \cmidrule(r){2-3} \cmidrule(r){6-6} \cmidrule(r){7-8} \cmidrule(r){9-10}
Language & Train & Test & \# Classes & $\ent(T)$ & $\ent(T \mid R)$  & $\ent(T \mid \vc(R))$ & $\calG(T, R, \vc)$ & $\ent(T \mid \vc(R))$ & $\calG(T, R, \vc)$ \\
    \midrule
Basque & $\phantom{0}\phantom{0}67{,}578$ & $\phantom{0}22{,}575$ & 29 & 4.03 & 0.62 & 0.75 & 0.13 (\phantom{0}3.1\%) & 1.39 & 0.77 (19.0\%) \\
Czech & $1{,}104{,}787$ & $163{,}770$ & 42 & 4.24 & 0.42 & 0.59 & 0.17 (\phantom{0}4.1\%) & 0.97 & 0.55 (13.1\%) \\
English & $\phantom{0}192{,}042$ & $\phantom{0}23{,}019$ & 48 & 4.48 & 0.45 & 1.00 & 0.55 (12.2\%) & 1.35 & 0.89 (19.9\%) \\
Finnish & $\phantom{0}150{,}362$ & $\phantom{0}19{,}515$ & 44 & 4.42 & 0.62 & 0.72 & 0.10 (\phantom{0}2.2\%) & 1.77 & 1.15 (26.0\%) \\
Indonesian & $\phantom{0}\phantom{0}93{,}054$ & $\phantom{0}11{,}223$ & 30 & 4.16 & 0.77 & 1.13 & 0.36 (\phantom{0}8.6\%) & 1.52 & 0.75 (18.0\%) \\
Korean & $\phantom{0}273{,}436$ & $\phantom{0}26{,}079$ & 30 & 4.17 & 0.40 & 0.76 & 0.36 (\phantom{0}8.7\%) & 1.50 & 1.10 (26.4\%) \\
Marathi & $\phantom{0}\phantom{0}\phantom{0}2{,}624$ & $\phantom{0}\phantom{0}\phantom{0}365$ & 39 & 4.01 & 1.39 & 1.65 & 0.26 (\phantom{0}6.5\%) & 2.26 & 0.87 (21.6\%) \\
Tamil & $\phantom{0}\phantom{0}\phantom{0}5{,}929$ & $\phantom{0}\phantom{0}1{,}869$ & 28 & 3.78 & 1.17 & 1.23 & 0.06 (\phantom{0}1.6\%) & 2.44 & 1.27 (33.7\%) \\
Telugu & $\phantom{0}\phantom{0}\phantom{0}4{,}031$ & $\phantom{0}\phantom{0}\phantom{0}575$ & 41 & 3.64 & 1.09 & 1.31 & 0.23 (\phantom{0}6.2\%) & 1.85 & 0.76 (20.9\%) \\
Turkish & $\phantom{0}\phantom{0}34{,}120$ & $\phantom{0}\phantom{0}9{,}046$ & 31 & 3.95 & 1.12 & 1.17 & 0.05 (\phantom{0}1.2\%) & 2.01 & 0.89 (22.4\%) \\
Urdu & $\phantom{0}104{,}647$ & $\phantom{0}14{,}271$ & 24 & 3.83 & 0.63 & 0.93 & 0.30 (\phantom{0}8.0\%) & 1.08 & 0.46 (11.9\%) \\
    \bottomrule 
    \end{tabular}
    }
    \caption{Amount of information $\bert$, fastText or one-hot embeddings share with a dependency arc labeling task. 
    $\ent(T)$ is again estimated with a plug-in estimator from same treebanks we use to train our models.
    } \vspace{-.5em}
    \label{tab:results-dependency}
\end{table*}

\subsection{Results}
We know $\bert$ can generate text in many languages. 
Here we assess how much it actually ``knows'' about syntax in those languages---or at least how much we can extract from it given as powerful probes as we can train. We further evaluate how much it knows above and beyond simple type-level baselines.

\paragraph{POS tags} \cref{tab:results-full} presents these results, showing how much information $\bert$, fastText, and one-hot embeddings encode about POS tagging.
We see that---in all analysed languages---type level embeddings can already capture most of the uncertainty in POS tagging.
We also see that BERT only shares a small amount of extra information with the task, having small gains in all languages---$\bert$ even presents negative gains in some of them. Although this may seem to contradict the information processing inequality, it is actually caused by the difficulty of approximating $\id$ and $\vc(\cdot)$ with a finite training set---causing $\KLqthetaone(T \mid R)$ to be larger than $\KLqthetatwo(T \mid \vc(R))$.
This highlights the need to formalize \emph{ease of extraction}, as discussed in \cref{sec:ease-extract}.

\paragraph{Dependency labels} As shown in \cref{tab:results-dependency}, $\bert$ improves over type-level embeddings in all languages on this task. Nonetheless, although this is a much more context-dependent task, we see $\bert$-based estimates reveal at most $12\%$ more information than fastText in English, the highest resource language in our set. If we look at the lower-resource languages, in five of them the gains are of less than $5\%$.

\paragraph{Discussion} 
When put into perspective, multilingual $\bert$'s representations do not seem to encode much more information about syntax than a simple baseline. On POS labeling, $\bert$ only improves upon fastText in five of the eleven analysed languages---and by small amounts (less than $9\%$) when it does. 
Even at dependency labelling, a task considered to require more contextual knowledge, we could only decode from $\bert$ at most (in English) $12\%$ additional information--- which again highlights the need to formalize ease of extraction.

\section{Conclusion}
We propose an information-theoretic operationalization of probing that defines it as the task of estimating conditional mutual information.
We introduce control functions, which put in context our mutual information estimates---how much more informative are contextual representations than some knowledge judged to be trivial? 
We further explored our operationalization and showed that, given perfect probes, probing can only yield insights into the language itself and cannot tell us anything about the representations under investigation.
Keeping this in mind, we suggest a change of focus---instead of concentrating on probe size or information, we should pursue \emph{ease of extraction} going forward.

On a final note, we apply our formalization to evaluate multilingual $\bert$'s syntactic knowledge on a set of eleven typologically diverse languages.
Although it does encode a large amount of information about syntax---more than $76\%$ and $65\%$, respectively, about POS and dependency labels in all languages\footnote{This is measured as the relative difference between $\ent(T)$ and $\ent(T\mid R)$. On average, this value is $88\%$ and $80\%$ on POS and dependency labels, respectively.}---$\bert$ only encodes at most $12\%$ more information than a simple baseline (a type-level representation). 
On POS labeling, more specifically, our MI estimates based on $\bert$ are higher than the control in less than half of the analyzed languages. 
This indicates that word-level POS labeling may not be ideal for contemplating the syntax contained in contextual word embeddings. 
\section*{Acknowledgements} 
The authors would like to thank Adam Poliak and John Hewitt for several helpful suggestions. 
\bibliography{acl2020}
\bibliographystyle{acl_natbib}

\clearpage
\appendix

\section{Variational Bounds}\label{appendix:a}

\begin{figure*}
    \centering
\begin{align}
    \calG(T, R, \ve) &:= \ent(T ; \vc(R)) - \ent(T \mid R) \label{eq:estimate_error}\\
    &= \entqthetatwo(T \mid \vc(R)) - \underset{\vr \sim p(\cdot)}{\mathbb{E}} \KL(p(\cdot \mid \vc(\vr)) \mid \mid \qthetatwo(\cdot \mid \vc(\vr))) \nonumber \\
    & \qquad - \entqthetaone(T \mid R) + \underset{\vr \sim p(\cdot)}{\mathbb{E}} \KL(p(\cdot \mid \vr) \mid \mid \qthetaone(\cdot \mid \vr)) \nonumber\\
    &= \entqthetatwo(T \mid \vc(R)) - \KLqthetatwo(T, \vc(R)) - \entqthetaone(T \mid R) + \KLqthetaone(T \mid R) \nonumber \\
    &= \entqthetatwo(T \mid \vc(R)) - \entqthetaone(T \mid R) + \KLqthetaone(T \mid R) - \KLqthetatwo(T, \vc(R)) \nonumber \\ 
    &= \underbrace{\calG_{\qtheta}(T, R, \ve)}_{\textit{estimated gain}} + \underbrace{\KLqthetaone(T \mid R) - \KLqthetatwo(T, \vc(R))}_{\textit{estimation error}} \nonumber
\end{align}
\caption{Derivation of the estimation error.}\label{fig:estimation_error}
\end{figure*}

\begin{thm}\label{thm:variational_bounds}
The estimation error between $\calG_{\qtheta}(T, R, \ve)$ and the true gain can be upper- and lower-bounded by two distinct Kullback--Leibler divergences.
\end{thm}
\begin{proof}
We first find the error given by our estimate, which is a difference between two KL divergences---as shown in \cref{eq:estimate_error} in \cref{fig:estimation_error}.
Making use of this error, we trivially find an upper-bound on the estimation error as
\begin{align}
    \Delta &= \KLqthetaone(T \mid R) - \KLqthetatwo(T, \vc(R)) \\
     &\le \KLqthetaone(T \mid R) \nonumber
\end{align}
which follows since KL divergences are never negative. Analogously, we find a lower-bound as
\begin{align}
    \Delta &= \KLqthetaone(T \mid R) - \KLqthetatwo(T, \vc(R)) \\
     &\ge - \KLqthetatwo(T, \vc(R)) \nonumber
\end{align}

\end{proof} 

\section{Further Results}\label{appendix:further_results}

In this section, we present accuracies for the models trained using $\bert$, fastText and one-hot embeddings, and the full results on random embeddings. These random embeddings are generated once before the task, at the type level, and kept fixed without training. \cref{tab:results-extra_pos} shows that both BERT and fastText present high accuracies at POS labeling in all languages, except Tamil and Marathi.
One-hot and random results are considerably worse, as expected, since they could not do more than take random guesses (e.g. guessing the most frequent label in the training test) in any word which was not seen during training.  \cref{tab:results-extra_dep} presents similar results for dependency labeling, although accuracies for this task are considerably lower.

These tables also show how ambiguous the linguistic task is given the word types ($\ent(T \mid \id(R))$). These values were calculated using a plug-in estimator on the treebanks---which are known to underestimate entropies when used in undersampled regimes \citep{archer2014bayesian}---so they should not be considered as good approximations. Even so, we can see that most of the analysed languages are not very ambiguous with respect to POS labeling, and that there is a large variability of uncertainty across languages with respect to both tasks.

\begin{table*}[h]
    \centering
    \resizebox{\textwidth}{!}{%
    \begin{tabular}{l l l l l l l l l}
    \toprule
 & \multicolumn{4}{c}{\textbf{accuracies}} & \multicolumn{2}{c}{\textbf{base entropies}} & \multicolumn{2}{c}{\textbf{random}} \\ \cmidrule(r){2-5} \cmidrule(r){6-7} \cmidrule(r){8-9}
Language & $\bert$ & fastText & one-hot & random & $\ent(T)$ & $\ent(T \mid \id(R))$ & $\ent(T \mid \vc(R))$ & $\calG(T, R, \vc)$ \\ 
    \midrule
Basque & 0.92 & 0.93 & 0.82 & 0.82 & 3.17 & 0.13 & 0.83 & 0.48 (15.0\%) \\
Czech & 0.98 & 0.98 & 0.91 & 0.87 & 3.33 & 0.06 & 0.57 & 0.47 (14.0\%) \\
English & 0.95 & 0.90 & 0.85 & 0.83 & 3.61 & 0.26 & 0.72 & 0.48 (13.4\%) \\
Finnish & 0.95 & 0.96 & 0.82 & 0.81 & 3.17 & 0.06 & 0.87 & 0.62 (19.6\%) \\
Indonesian & 0.92 & 0.92 & 0.86 & 0.84 & 3.24 & 0.16 & 0.68 & 0.30 (\phantom{0}9.2\%) \\
Korean & 0.92 & 0.85 & 0.73 & 0.70 & 3.04 & 0.14 & 1.33 & 1.01 (33.1\%) \\
Marathi & 0.83 & 0.79 & 0.68 & 0.69 & 3.17 & 0.48 & 1.43 & 0.67 (21.1\%) \\
Tamil & 0.88 & 0.89 & 0.64 & 0.68 & 3.15 & 0.09 & 1.41 & 0.82 (26.2\%) \\
Telugu & 0.91 & 0.92 & 0.78 & 0.82 & 2.73 & 0.07 & 0.86 & 0.44 (16.2\%) \\
Turkish & 0.92 & 0.95 & 0.79 & 0.80 & 3.03 & 0.08 & 0.81 & 0.45 (14.7\%) \\
Urdu & 0.92 & 0.91 & 0.88 & 0.87 & 3.23 & 0.29 & 0.59 & 0.27 (\phantom{0}8.3\%) \\
    \bottomrule 
    \end{tabular}
    }
    \caption{Accuracies of the models trained on $\bert$, fastText, one-hot and random embeddings for the POS tagging task.}
    \label{tab:results-extra_pos}
\end{table*}

\begin{table*}[h]
    \centering
    \resizebox{\textwidth}{!}{%
    \begin{tabular}{l l l l l l l l l}
    \toprule
 & \multicolumn{4}{c}{\textbf{accuracies}} & \multicolumn{2}{c}{\textbf{base entropies}} & \multicolumn{2}{c}{\textbf{random}} \\ \cmidrule(r){2-5} \cmidrule(r){6-7} \cmidrule(r){8-9}
Language & $\bert$ & fastText & one-hot & random & $\ent(T)$ & $\ent(T \mid \id(R))$ & $\ent(T \mid \vc(R))$ & $\calG(T, R, \vc)$ \\ 
    \midrule
Basque & 0.87 & 0.83 & 0.71 & 0.65 & 4.03 & 0.55 & 1.71 & 1.08 (26.9\%) \\
Czech & 0.91 & 0.88 & 0.80 & 0.68 & 4.24 & 0.78 & 1.58 & 1.16 (27.3\%) \\
English & 0.91 & 0.78 & 0.72 & 0.68 & 4.48 & 1.01 & 1.61 & 1.16 (25.8\%) \\
Finnish & 0.87 & 0.85 & 0.65 & 0.56 & 4.42 & 0.52 & 2.21 & 1.59 (36.1\%) \\
Indonesian & 0.85 & 0.76 & 0.69 & 0.64 & 4.16 & 0.83 & 1.76 & 0.99 (23.9\%) \\
Korean & 0.92 & 0.84 & 0.68 & 0.56 & 4.17 & 0.35 & 2.08 & 1.68 (40.4\%) \\
Marathi & 0.75 & 0.70 & 0.61 & 0.62 & 4.01 & 0.81 & 2.12 & 0.73 (18.2\%) \\
Tamil & 0.76 & 0.74 & 0.51 & 0.54 & 3.78 & 0.31 & 2.32 & 1.15 (30.5\%) \\
Telugu & 0.80 & 0.78 & 0.67 & 0.69 & 3.64 & 0.31 & 1.96 & 0.88 (24.1\%) \\
Turkish & 0.77 & 0.75 & 0.59 & 0.54 & 3.95 & 0.54 & 2.14 & 1.02 (25.8\%) \\
Urdu & 0.87 & 0.80 & 0.76 & 0.73 & 3.83 & 1.02 & 1.26 & 0.63 (16.4\%) \\
    \bottomrule 
    \end{tabular}
    }
    \caption{Accuracies of the models trained on $\bert$, fastText, one-hot and random embeddings for the dependency labeling task.}
    \label{tab:results-extra_dep}
\end{table*}

\end{document}